\documentclass{article}
\usepackage{nips14submit_e,times}
\usepackage[disable]{todonotes}
\usepackage{graphicx} 
\usepackage{subfigure}
\usepackage{amssymb}
\usepackage{xspace}

\usepackage{algorithm}
\usepackage{algorithmic}

\usepackage{amssymb,amsmath,amsthm,amsbsy,amsfonts}
\usepackage{color} 

\newtheorem{theorem}{Theorem}[section]
\newtheorem{lemma}[theorem]{Lemma}

\DeclareMathOperator*{\Exp}{\mathbb{E}} %

\DeclareMathOperator*{\argmin}{arg\,min}

\title{Reinforcement and Imitation Learning\\via Interactive No-Regret Learning}

\author{
St\'ephane Ross ~~~~ J. Andrew Bagnell \\
\texttt{stephaneross@cmu.edu} ~~~~ \texttt{dbagnell@ri.cmu.edu} \\
The Robotics Institute \\ 
Carnegie Mellon University, \\ Pittsburgh, PA, USA
}

\newcommand{\DAGGER}{\textsc{DAgger}\xspace}

\newcommand{\aggrevate}{\textsc{AggreVaTe}\xspace}
\newcommand{\ctgd}{\textsc{AggreVaTe}\xspace}

\newcommand{\rldagger}{\textsc{NRPI}\xspace}

\nipsfinalcopy
\begin{document}

\maketitle
\begin{abstract}
Recent work has demonstrated that problems-- particularly imitation learning and structured prediction-- where a  learner's predictions influence the input-distribution it is tested on can be naturally addressed by an interactive approach and analyzed using no-regret online learning.  These approaches to imitation learning, however, neither require nor benefit from information about the cost of actions. We extend existing results in two directions: first, we develop an interactive imitation learning approach that leverages cost information; second, we extend the technique to address reinforcement learning. The results provide theoretical support to the commonly observed successes of online approximate policy iteration. Our approach suggests a broad new family of algorithms and provides a unifying view of existing techniques for imitation and reinforcement learning.
\end{abstract} 

\section{Introduction}
	
	Imitation learning has become increasingly important in fields-- notably robotics and game AI-- where it is easier for an expert to demonstrate a behavior than to translate that behavior to code. \cite{Argall09}
	Perhaps surprisingly, it has also become central in developing predictors for complex output spaces, e.g. sets and lists \cite{Ross12}, parse trees \cite{Daume09}, image parsing \cite{Munoz10,Ross11b} and natural language understanding\cite{Duvallet13}. In these domains, a policy is trained to imitate an oracle on ground-truthed data. Iterative training procedures  (\emph{e.g.} \DAGGER, SEARN, \textsc{SMILe}\cite{Ross11,Daume09,Ross10}) that interleave policy execution and learning 
	have demonstrated impressive practical performance and strong theoretical guarantees that were not possible with batch supervised learning. Most of these approaches to imitation learning, however, neither require nor benefit from information about the cost of actions; rather they leverage only information provided about ``correct'' actions by the demonstrator.

	While iterative training corrects the compounding of error effect one sees
	in control and decision making applications, it does not address all issues
	that arise. Consider, for instance, a problem of learning to drive near the
	edge of a cliff: methods like \DAGGER consider all errors from
	agreeing with the expert driver equally. If driving immediately off the
	cliff makes the expert easy to imitate-- because the expert simply chooses
	the go straight from then on-- these approaches may learn that very poor
	policy. More generally, a method that only reasons about agreement with a
	demonstrator instead of the long term costs of various errors may poorly
	trade-off inevitable mistakes. Even a crude estimate of the cost-to-go
	(\emph{e.g.} it's very expensive to drive off the cliff)-- may improve a
	learned policy's performance at the user's intended task.

	 SEARN, by contrast, \textbf{does} reason about cost-to-go, but uses rollouts from the current policy which can be impractical for imitation learning. SEARN additionally requires the use of stochastic policies.

	We present a simple, general approach we term \aggrevate (Aggregate Values to Imitate) that leverages cost-to-go information in addition to correct demonstration, and establish that previous methods can be understood as special cases of a more general no-regret strategy. The approach provides much stronger guarantees than existing methods by providing a statistical \emph{regret} rather then statistical \emph{error} reduction. \cite{Beygelzimer2008}

	This general strategy of leveraging cost-sensitive no-regret learners can be extended to Approximate Policy Iteration (API) variants for reinforcement learning. We show that any no-regret learning algorithm can be used to develop stable API algorithms with guarantees as strong as any available in the literature. We denote the resulting algorithm \rldagger. The results provide theoretical support to the commonly observed success of online policy iteration\cite{sutton00} despite a paucity of formal results: such online algorithms often enjoy no-regret guarantees or share similar stability properties.  Our approach suggests a broad new family of algorithms and provides a unifying view of existing techniques for both imitation and reinforcement learning.

\section{Imitation Learning with Cost-To-Go}

\subsection{Preliminaries}

We consider in this work finite horizon\footnote{All our results can be easily extended to the infinite discounted horizon setting.} control problems in the form of a
Markov Decision Process with states $s$ and actions $a$. We assume the
existence of a cost-function $C(s,a)$, bounded between $0$ and $1$, that we are attempting to optimize over a horizon of $T$ decisions. We denote a class of policies $\Pi$ mapping states \footnote{More generally features of the state (and potentially time)-- our derivations do not require full observability and hence carry over to featurized state of POMDPs.} to actions.

We use $Q^\pi_t(s,a)$ to denote the expected future cost-to-go of executing action $a$ in state $s$, followed by executing policy $\pi$ for ${t-1}$ steps. We denote by $d_\pi = \frac{1}{T}\sum_{t=1}^{T} d_\pi^t$ the time-averaged distribution over states induced by executing policy $\pi$ in the MDP ($d_\pi^t$ is the distribution of states at time $t$ induced by executing policy $\pi$).  The overall performance metric of total cost of executing $\pi$ for $T$-steps is denoted  
$J(\pi) = \sum_{t=1}^{T} \Exp_{s \sim d^t_{\pi}} [ C(s,\pi(s) ) ]$
We assume system dynamics are either unknown or complex enough that we typically have only sample access to them. The resulting setting for learning policies by demonstration-- or learning policies by approximate policy iteration-- are \textbf{not} typical \emph{i.i.d.} supervised learning problems as the learned policy strongly influences its own test distribution rendering optimization difficult.

\subsection{Algorithm: \ctgd}

We describe here a simple extension of the \DAGGER technique of \cite{Ross11} that learns to choose actions to minimize the cost-to-go of the expert rather than the zero-one classification loss of mimicking its actions. In simplest form, on the first iteration \ctgd collects data by simply observing the expert perform the task, and in each trajectory, at a uniformly random time $t$, explores an action $a$ in state $s$, and observes the cost-to-go $Q$ of the expert after performing this action. (See Algorithm \ref{alg:ctgd} below.) \footnote{This cost-to-go may be estimated by rollout, or provided by the expert.}

Each of these steps generates a cost-weighted training example $(s,t,a,Q)$ \cite{Mineiro10} and \ctgd trains a policy $\hat{\pi}_2$ to minimize the expected cost-to-go on this dataset. At each following iteration $n$, \ctgd collects data through interaction with the learner as follows: for each trajectory, begin by using the current learner's policy $\hat{\pi}_n$ to perform the task, interrupt at a uniformly random time $t$, explore an action $a$ in the current state $s$, after which control is provided back to the expert to continue up to time-horizon $T$. This results in new examples of the cost-to-go of the expert $(s,t,a,Q)$, under the distribution of states visited by the current policy $\hat{\pi}_n$. This new data is aggregated with all previous data to train the next policy $\hat{\pi}_{n+1}$;  more generally, this data can be used by a no-regret online learner to update the policy and obtain $\hat{\pi}_{n+1}$. This is iterated for some number of iterations $N$ and the best policy found is returned. We optionally allow the algorithm to continue executing the expert's actions with small probability $\beta_n$, instead of always executing $\hat{\pi}_n$, up to the random time $t$ where an action is explored and control is shifted to the expert. The general \ctgd is detailed in Algorithm \ref{alg:ctgd}.

Observing the expert's cost-to-go indicates how much cost we might expect to incur in the future if we take an action now and then can behave as well (or nearly so) as the expert henceforth. Under the assumption that the expert is a good policy, and that the policy class $\Pi$ contains similarly good policies, this provides a rough estimate of what good policies in $\Pi$ will be able to achieve at future steps. By minimizing this cost-to-go at each step, we will choose policies that lead to situations where incurring low future cost-to-go is possible. For instance, we will be able to observe that if some actions put the expert in situations where falling off a cliff or crash is inevitable then these actions must be avoided at all costs in favor of those where the expert is still able to recover.

\begin{algorithm}
\begin{algorithmic}
\STATE Initialize $\mathcal{D} \leftarrow \emptyset$, $\hat{\pi}_1$ to any policy in $\Pi$.
\FOR{$i=1$ \textbf{to} $N$}
\STATE Let $\pi_i = \beta_i \pi^* + (1-\beta_i) \hat{\pi}_{i}$ \#Optionally mix in expert's own behavior.
\STATE Collect $m$ data points as follows:
\FOR{$j=1$ \textbf{to} $m$}
\STATE Sample uniformly $t \in \{1,2,\dots,T\}$.
\STATE Start new trajectory in some initial state drawn from initial state distribution
\STATE Execute current policy $\pi_i$ up to time $t-1$.
\STATE Execute some exploration action $a_t$ in current state $s_t$ at time $t$
\STATE Execute expert from time $t+1$ to $T$, and observe estimate of cost-to-go $\hat{Q}$ starting at time $t$
\ENDFOR
\STATE Get dataset $\mathcal{D}_i = \{ (s, t, a, \hat{Q}) \}$ of states, times, actions, with expert's cost-to-go.
\STATE Aggregate datasets: $\mathcal{D} \leftarrow \mathcal{D} \bigcup \mathcal{D}_i$.
\STATE Train cost-sensitive classifier $\hat{\pi}_{i+1}$ on $\mathcal{D}$\\
 ~~~~ (\emph{Alternately: use any online learner on the data-sets $\mathcal{D}_i$ in sequence to get $\hat{\pi}_{i+1}$ })
\ENDFOR
\STATE \textbf{Return} best $\hat{\pi}_i$ on validation.
\end{algorithmic}
\caption{\ctgd : Imitation Learning with Cost-To-Go \label{alg:ctgd}}
\end{algorithm}

In \ctgd the problem of choosing the sequence of policies $\hat{\pi}_1, \hat{\pi}_2, \dots, \hat{\pi}_N$ over iterations is viewed as an online cost-sensitive classification problem. Our analysis below demonstrates that any no-regret algorithm on such problems can be used to update the sequence of policies and provide strong guarantees. To achieve this, when the policy class $\Pi$ is finite, randomized online learning algorithms like \emph{weighted majority} \cite{CesaBianchi06} may be used. When dealing with infinite policy classes (\emph{e.g.} all linear classifiers), no-regret online cost-sensitive classification is not always computationally tractable. Instead, typically reductions of cost-sensitive classification to regression or ranking problems as well as convex upper bounds \cite{Beygelzimer2008} on the classification loss lead to efficient no-regret online learning algorithms (\emph{e.g.} gradient descent).
The algorithm description suggests as the ``default'' learning strategy a \emph{(Regularized)-Follow-The-Leader} online learner:
it attempts to learn a good classifier for \textbf{all} previous data. This strategy
for certain loss function (notably strongly convex surrogates to the cost-sensitive classification loss) \cite{CesaBianchi06} and any sufficiently stable batch learner \cite{Ross12b,Saha2012interplay} ensures the no-regret property. It also highlights why the approach is likely to be particularly stable across rounds of interaction.

\subsection{Training the Policy to Minimize Cost-to-Go}
In standard ``full-information'' cost-sensitive classification, a cost vector is provided for each data-point in the training data that indicates the cost of predicting each class or label for this input.
In our setting, that implies for each sampled state we recieve a cost-to-go estimate/rollout for all actions. Training the policy at each iteration then simply corresponds to solving a cost-sensitive classification problem. That is, if we collect a dataset of $m$ samples, $\{(s_i,t_i,\hat{Q}_i)\}_{i=1}^m$, where $\hat{Q}_i$ is a cost vector of cost-to-go estimates for each action in state $s_i$ at time $t_i$, then we solve the cost-sensitive classification problem: $\argmin_{\pi \in \Pi} \sum_{i=1}^m \hat{Q}_i(\pi(s_i,t_i))$. Reductions of cost-sensitive classification to convex optimization problems can be used like weighted multi-class Support Vector Machines or ranking\cite{Beygelzimer2008}, to obtain problems that can be optimized efficiently while still guaranteeing good performance at this cost-sensitive classification task. 

For instance, a simple approach is to transform this into an \emph{argmax regression problem}: i.e., $\pi_n(s,t) = \argmin_{a \in A} Q_n(s,t,a)$, for $Q_n$ the learned regressor at iteration $n$ that minimizes the squared loss at predicting the cost-to-go estimates: $Q_n = \argmin_{Q \in \mathcal{Q}} \sum_{(s_i,t_i,a_i,\hat{Q}_i) \in D}(Q(s_i,t_i,a_i) - \hat{Q}_i)^2$, where $D$ is the dataset of all collected cost-to-go estimates so far, and $\mathcal{Q}$ the class of regressors considered (e.g. linear regressors). This approach also naturally handles
the more common situation in imitation learning where we only have \emph{partial information} for a particular action chosen at a state. Alternate approaches include importance weighting techniques to transform the problem into a standard cost-sensitive classification problem \cite{Horvitz52,Dudik11} and other online learning approaches meant to handle
``bandit'' feedback.

\paragraph{Local Exploration in Partial Information Setting} In the \emph{partial information} setting we must also
select which action to explore for an estimate of cost-to-go.
The uniform strategy is simple and effective but inefficient. The problem
may be cast as a \emph{contextual bandit problem} \cite{Auer03,Beygelzimer11}
where features of the current state define the context of exploration. These algorithms, by choosing more carefully than at random, may be significantly more sample efficient. In our setting, in contrast with traditional bandit settings, we care only about the final learned policy and not the cost of explored actions along the way. Recent work \cite{Avner12} may be more appropriate for improving performance in this case. Many contextual bandit algorithms require a finite set of policies $\Pi$ \cite{Auer03} or full realizability \cite{Li10}, and this is an open and very active area of research that could have many applications here.

\subsection{Analysis}

We analyze \ctgd, showing that the no-regret property of online learning procedures can be leveraged in this interactive learning procedure to obtain strong performance guarantees. Our analysis seeks to answer the following question: how well does the learned policy perform if we can repeatedly identify good policies that incur cost-sensitive classification loss competitive with the expert demonstrator on the aggregate dataset we collect during training?

Our analysis of \ctgd relies on connecting the iterative learning procedure with the (adversarial) online learning problem \cite{CesaBianchi06} and using the no-regret property of the underlying cost-sensitive classification algorithm choosing policies $\hat{\pi}_{1:N}$. Here, the online learning problem is defined as follows: at each iteration $i$, the learner chooses a policy $\hat{\pi}_i \in \Pi$ that incurs loss  $\ell_i$ chosen by the adversary, and defined as $\ell_i(\pi) = \Exp_{t \sim U(1:T), s \sim d^t_{\pi_i}}[Q_{T-t+1}^*(s,\pi)]$ for $U(1:T)$ the uniform distribution on the set $\{1,2,\dots,T\}$, $\pi_i = \beta_i \pi^* + (1-\beta_i)\hat{\pi}_i$ and $Q^*$ the cost-to-go of the expert. We can see that \ctgd at iteration $i$ is exactly collecting a dataset $\mathcal{D}_i$, that provides an empirical estimate of this loss $\ell_i$. 

Let $\epsilon_{\textrm{class}} = \min_{\pi \in \Pi} \frac{1}{N} \sum_{i=1}^N \Exp_{t \sim U(1:T), s \sim d^t_{\pi_i}}[ Q_{T-t+1}^*(s,a) - \min_a Q_{T-t+1}^*(s,a) ]$ denote the minimum expected cost-sensitive classification regret achieved by policies in the class $\Pi$ on all the data over the $N$ iterations of training. Denote the online learning average regret of the sequence of policies chosen by \ctgd, $\epsilon_{\textrm{regret}} = \frac{1}{N} [ \sum_{i=1}^N \ell_i(\hat{\pi}_i) - \min_{\pi \in \Pi}  \sum_{i=1}^N \ell_i(\pi) ]$. 

We provide guarantees for the ``uniform mixture'' policy $\overline{\pi}$, that at the beginning of any trajectory samples a policy $\pi$ uniformly randomly among the policies $\{ \hat{\pi}_i \}_{i=1}^N$ and executes this policy $\pi$ for the entire trajectory. This immediately implies good performance for the best policy $\hat{\pi}$ in the sequence $\hat{\pi}_{1:N}$, i.e. $J(\hat{\pi}) = \min_{i \in 1:N} J(\hat{\pi}_i) \leq J(\overline{\pi})$, and the last policy $\hat{\pi}_N$ when the distribution of visited states converges over the iterations of learning.

Assume the cost-to-go of the expert $Q^*$ is non-negative and bounded by $Q^*_{\max}$, and $\beta_i \leq (1-\alpha)^{i-1}$ for all $i$ for some constant $\alpha$ \footnote{The default parameter-free version of \ctgd corresponds to $\alpha=1$, using $0^0 = 1$.}. Then the following holds in the infinite sample case (\emph{i.e.} if at each iteration of \ctgd we collected an arbitrarily large amount of data by running the current policy):

\begin{theorem}\label{thmDaggerCostToGo}
After $N$ iterations of \ctgd:
\begin{displaymath}
J(\hat{\pi}) \leq J(\overline{\pi}) \leq J(\pi^*) + T[\epsilon_{\textrm{class}} + \epsilon_{\textrm{regret}}] + O\left(\frac{Q_{\max} T \log T}{\alpha N}\right).
\end{displaymath}
Thus if a no-regret online algorithm is used to pick the sequence of policies $\hat{\pi}_{1:N}$, then as the number of iterations $N \rightarrow \infty$:
\begin{displaymath}
\lim_{N \rightarrow \infty} J(\overline{\pi}) \leq J(\pi^*) + T \epsilon_{\textrm{class}}
\end{displaymath}
\end{theorem}
\todo{Add proof for simple case?}
The proof of this result is presented in the Appendix. This theorem indicates that after sufficient iterations, \ctgd will find policies that perform the task nearly as well as the demonstrator if there are policies in $\Pi$ that have small cost-sensitive classification regret on the aggregate dataset (\emph{i.e.} policies with cost-sensitive classification loss not much larger than that of the bayes-optimal one on this dataset). Note that non-interactive supervised learning methods are unable to achieve a similar bound which degrades only linearly with $T$ and the cost-sensitive classification regret. \cite{Ross11}. 

The analysis above abstracts away the issue of action exploration and learning from finite data. These issues come into play in a sample complexity analysis. Such analyses depend on many factors such as the particular reduction and exploration method. When reductions of cost-sensitive classification to simpler regression/ranking/classification \cite{Beygelzimer2005} problems are used, our results can directly relate the task performance of the learned policy to the performance on the simpler problem. 
To illustrate how such results may be derived, we provide a result for the special case where actions are explored uniformly at random and the reduction of cost-sensitive classification to regression is used. 

In particular, if $\hat{\epsilon}_{\textrm{regret}}$ denotes the empirical average online learning regret on the training regression examples collected over the iterations, and $\hat{\epsilon}_{\textrm{class}}$ denotes the empirical regression regret of the best regressor in the class on the aggregate dataset of regression examples when compared to the bayes-optimal regressor, 
we have that:
\begin{theorem}\label{thmDaggerCostToGoGen}
$N$ iterations of \ctgd, collecting $m$ regression examples $(s,a,t,Q)$ per iteration, guarantees that with probability at least 1-$\delta$:
\begin{displaymath}
J(\hat{\pi}) \leq J(\overline{\pi}) \leq J(\pi^*) + 2 \sqrt{|A|}T \sqrt{\hat{\epsilon}_{\textrm{class}} + \hat{\epsilon}_{\textrm{regret}} + O(\sqrt{\log(1/\delta)/Nm})} + O\left(\frac{Q_{\max} T \log T}{\alpha N}\right).
\end{displaymath}
Thus if a no-regret online algorithm is used to pick the sequence of regressors $\hat{Q}_{1:N}$, then as the number of iterations $N \rightarrow \infty$, with probability $1$:
\begin{displaymath}
\lim_{N \rightarrow \infty} J(\overline{\pi}) \leq J(\pi^*) + 2 \sqrt{|A|}T \sqrt{\hat{\epsilon}_{\textrm{class}}}
\end{displaymath}
\end{theorem}
The detailed proof is presented in the Appendix. This result demonstrates how the task performance of the learned policies may be related all the way down to the regret on the regression loss at predicting the observed cost-to-go during training. In particular, it relates task performance to the square root of the online learning regret, on this regression loss, and the regression regret of the best regressor in the class to the bayes-optimal regressor on this training data. \footnote{The appearance of the square root is particular to the use of this reduction to squared-loss regression and implies relative slow convergence to good performance. 
Other cost-sensitive classification reductions and regression losses (\emph{e.g.} \cite{Langford05SECOC,Beygelzimer09ECT}) do not introduce this square root and still allow efficient learning.}

\subsection{Discussion}

\todo{JAB: this paragraph needs work.}
\paragraph{\ctgd as a reduction:} \ctgd can be interpreted as a \emph{regret} reduction of imitation learning to no-regret online learning. \footnote{Unfortunately \emph{regret} here has two different meanings common in the literature: the first is in the statistical sense of doing nearly as well as the Bayes-optimal predictor. \cite{Beygelzimer2008} The second use is in the online, adversarial, no-regret sense of competing against any hypothesis on a particular sequence without statistical assumptions. \cite{CesaBianchi06} } We present a statistical regret reduction, as here, performance is related directly to the online, cost-sensitive classification regret on the aggregate dataset. By minimizing cost-to-go, we obtain regret reduction, rather than a weaker error reduction as in \cite{Ross11} when simply minimizing immediate classification loss.

\paragraph{Limitations:} As just mentioned, in cases where the expert is much better than any policy in $\Pi$, the expert's cost-to-go may be a very optimistic estimate of the true future cost after taking a certain action. The approach may fail to learn policies that perform well, even if policies that can perform the task (albeit not as well as the expert) exist in the policy class. Consider again the driving scenario, where one may choose one of two roads to reach a goal: a shorter route that involves driving on a very narrow road next to cliffs on either side, and a longer route which is safer and risks no cliff. If in this example, the expert takes the short route faster and no policy in the class $\Pi$ can drive without falling on the narrow road, but there exists policies that can take the longer road and safely reach the goal, this algorithm would fail to find these policies. The reason for this is that, as we minimize cost-to-go of the expert, we would always favor policies that heads toward the shorter narrow road. But once we are on that road, inevitably at some point we will encounter a scenario where no policies in the class can predict the same low cost-to-go actions as the expert (i.e. making $\epsilon$ large in the previous guarantee). The end result is that we may learn a policy that takes the short narrow road and eventually falls off the cliff, in these pathological scenarios. 

\paragraph{Comparison to SEARN:} \ctgd shares deep commonalities with SEARN but by providing a reduction to online learning allows much more general schemes to update the policy at each iteration that may be more convenient or efficient rather than the particular stochastic mixing update of SEARN. These include deterministic ones that provide upper convex bounds on performance. In fact, SEARN may be thought as a particular case of \ctgd, where the policy class is the set of distributions over policies, and the online coordinate descent algorithm (Frank-Wolfe) of \cite{Hazan2012projection} is used to update the distribution over policies at each iteration.  Both collect data in a similar fashion at each iteration by executing the current policy up to a random time and then collecting cost-to-go estimates for explored actions in the current state. A distinction is that SEARN collects cost-to-go of the \textbf{current} policy after execution of the random action, instead of the cost-to-go of the expert. Interestingly, SEARN is usually used in practice with the approximation of collecting cost-to-go of the expert \cite{Daume09}, rather than the current policy. Our approach can be seen as providing a theoretical justification for what was previously a
heuristic. 


\section{Reinforcement Learning via No-Regret Policy Iteration} \label{secDAggerPSDP}

A relatively simple modification of the above approach enables us to develop a family of sample-based approximate policy iteration
algorithms. Conceptually, we make a swap: from executing the current policy and then switching to the expert to observe a cost-to-go; to, executing the expert policy while collecting cost-to-go of the learner's current policy. We denote this family of algorithms \emph{No-Regret Policy Iteration} \rldagger and detail and analyze it below.

This alternate has similar guarantees to the previous version, but may be preferred when no policy in the class is as good as the expert or when only a distribution of ``important states'' is available. In addition it can be seen to address a general model-free reinforcement learning setting where we simply have a state exploration distribution we can sample from and from which we collect examples of the current policy's cost-to-go. This is similar in spirit to how Policy Search by Dynamic Programming (PSDP) \cite{Bagnell03, Scherrer2014approximate} proceeds, and in some sense, the algorithm we present here provides a generalization of PSDP. However, by learning a stationary policy instead of a non-stationary policy, \rldagger can generalize across time-steps and potentially lead to more efficient learning and practical implementation in problems where $T$ is large or infinite. 

Following \cite{Bagnell03,Kakade02} we assume access to a state exploration distribution $\nu_t$ for all times $t$ in $1,2,\dots,T$. As will be justified by our theoretical analysis, these state exploration distributions should ideally be (close to) that of a (near-)optimal policy in the class $\Pi$. In the context where an expert is present, then this may simply be the distribution of states induced by the expert policy, i.e. $\nu_t = d^t_{\pi^*}$. In general, this may be the state distributions induced by some base policy we want to improve upon, or be determined from prior knowledge of the task.

Given the exploration distributions $\nu_{1:T}$, \rldagger proceeds as follows. At each iteration $n$, it collects cost-to-go examples by sampling uniformly a time $t \in \{1,2,\dots,T\}$, sampling a state $s_t$ from $\nu_t$, and then executes an exploration action $a$ in $s_t$ followed by execution of the current learner's policy $\pi_n$ for time $t+1$ to $T$, to obtain a cost-to-go estimate $(s,a,t,Q)$ of executing $a$ followed by $\pi_n$ in state $s$ at time $t$. \footnote{In the particular case where $\nu_t = d^t_\pi$ of an exploration policy $\pi$, then to sample $s_t$, we would simply execute $\pi$ from time $1$ to $t-1$, starting from the initial state distribution.} Multiple cost-to-go estimates are collected this way and added in dataset $D_n$. After enough data has been collected, we update the learner's policy, to obtain $\pi_{n+1}$, using any no-regret online learning procedure, on the loss defined by the cost-sensitive classification examples in the new data $D_n$. This is iterated for a large number of iterations $N$. Initially, we may start with $\pi_1$ to be any guess of a good policy from the class $\Pi$, or use the expert's cost-to-go at the first iteration, to avoid having to specify an initial policy. This algorithm is detailed in Algorithm \ref{algDaggerPSDP}.

\begin{algorithm}
\begin{algorithmic}
\STATE Initialize $\mathcal{D} \leftarrow \emptyset$, $\hat{\pi}_1$ to any policy in $\Pi$.
\FOR{$i=1$ \textbf{to} $N$}
\STATE Collect $m$ data points as follows:
\FOR{$j=1$ \textbf{to} $m$}
\STATE Sample uniformly $t \in \{1,2,\dots,T\}$.
\STATE Sample state $s_t$ from exploration distribution $\nu_t$.
\STATE Execute some exploration action $a_t$ in current state $s_t$ at time $t$
\STATE Execute $\hat{\pi}_i$ from time $t+1$ to $T$, and observe estimate of cost-to-go $\hat{Q}$ starting at time $t$
\ENDFOR
\STATE Get dataset $\mathcal{D}_i = \{ (s, a, t, \hat{Q}) \}$ of states, actions, time, with current policy's cost-to-go.
\STATE Aggregate datasets: $\mathcal{D} \leftarrow \mathcal{D} \bigcup \mathcal{D}_i$.
\STATE Train cost-sensitive classifier $\hat{\pi}_{i+1}$ on $\mathcal{D}$ 
 ~~~~ (\emph{Alternately: use any online learner on the data-sets $\mathcal{D}_i$ in sequence to get $\hat{\pi}_{i+1}$ })
\ENDFOR
\STATE \textbf{Return} best $\hat{\pi}_i$ on validation.
\end{algorithmic}
\caption{\rldagger Algorithm \label{algDaggerPSDP}}
\end{algorithm}

\subsection{Analysis}

Consider the loss function $L_n$ given to the online learning algorithm within \rldagger at iteration $n$. Assuming infinite data, it assigns the following loss to each policy $\pi \in \Pi$:
$$L_n(\pi) = \mathbb{E}_{t \sim U(1:T), s \sim \nu_t}[Q^{\hat{\pi}_n}_{T-t+1}(s,\pi)].$$
This loss represents the expected cost-to-go of executing $\pi$ immediately for one step followed by current policy $\hat{\pi}_n$, under the exploration distributions $\nu_{1:T}$.

This sequence of losses over the iterations of training corresponds to an online cost-sensitive classification problem, as in the previous \ctgd algorithm. Let $\epsilon_{\textrm{regret}}$ be the average regret of the online learner on this online cost-sensitive classification problem after $N$ iterations of \rldagger:
$$\epsilon_{\textrm{regret}} = \frac{1}{N} \sum_{i=1}^N L_i(\pi_i) - \min_{\pi \in \Pi} \frac{1}{N} \sum_{i=1}^N L_i(\pi).$$

For any policy $\pi \in \Pi$, denote the average $L_1$ or variational distance between $\nu_t$ and $d^t_{\pi}$ over time steps $t$ as $D(\nu,\pi) = \frac{1}{T} \sum_{t=1}^T ||\nu_t - d^t_\pi||_1.$
Note that if $\nu_t = d^t_\pi$ for all $t$, then $D(\nu,\pi) = 0$.

Denote by $Q_{\max}$ a bound on cost-to-go (which is always $\leq T C_{\max}$). Denote $\hat{\pi}$ the best policy found by \rldagger over  iterations, and $\overline{\pi}$ the uniform mixture policy over $\pi_{1:N}$ defined as before. Then \rldagger achieves the following guarantee:
\begin{theorem} \label{thmDaggerPSDP}
For any policy $\pi' \in \Pi$:
\begin{displaymath}
J(\hat{\pi}) \leq J(\overline{\pi}) \leq J(\pi') + T \epsilon_{\textrm{regret}} + T Q_{\max} D(\nu,\pi')
\end{displaymath}
If a no-regret online cost-sensitive classification algorithm is used:
$\lim_{N \rightarrow \infty} J(\overline{\pi}) \leq J(\pi') + T Q_{\max} D(\nu,\pi')$
\end{theorem}

\rldagger thus finds policies that are as good as any other policy $\pi' \in \Pi$ whose state distribution $d^t_{\pi'}$ is close to $\nu_t$ on average over time $t$.  Importantly, if $\nu_{1:T}$ corresponds to the state distribution of an optimal policy in class $\Pi$, then this theorem guarantees that \rldagger will find an optimal policy (within the class $\Pi$) in the limit.

This theorem provides a similar performance guarantee to the results for PSDP presented in \cite{Bagnell03}. \rldagger has the advantage of learning a single policy for test execution instead one at each time allowing for improved generalization and more efficient learning. \rldagger imposes stronger requirements: it uses a no-regret online cost-sensitive classification procedure instead of simply a cost-sensitive supervised learner. For finite policy classes $\Pi$, or using reductions of cost-sensitive classification as mentioned previously, we may still obtain convex online learning problems for which efficient no-regret strategies exist or use the simple aggregation of data-sets with any sufficiently stable batch learner. \cite{Ross12b,Saha2012interplay} 

The result presented here can be interpreted as a reduction of model-free reinforcement learning to no-regret online learning. It is a regret reduction, as performance is related directly to the online regret at the cost-sensitive classification task. However  performance is strongly limited by the quality of the exploration distribution. \footnote{One would naturally consider adapting the exploration distributions $\nu_{1:T}$ over the iterations of training. It can be shown that if $\nu^i_{1:T}$ are the exploration distributions at iteration $i$, and we have a mechanism for making $\nu^i_{1:T}$ converge to the state distributions of an optimal policy in $\Pi$ as $i \rightarrow \infty$, then we would always be guaranteed to find an optimal policy in $\Pi$. Unfortunately, no known method can guarantee this.}

\section{Discussion and Future Work}

\paragraph{Contribution.} The work here provides theoretical support for two seemingly unrelated empirical observations. First, and perhaps most crucially, much anecdotal evidence suggests that approximate policy iteration-- and especially online variants \cite{sutton00}-- is more effective and stable than theory and counter-examples to convergence might suggest. This cries out for some explanation; we contend that it can be understood as such online algorithms often enjoy no-regret guarantees or share similar stability properties than can ensure relative performance guarantees. 

Similarly, practical implementation of imitation learning-for-structured-prediction methods like SEARN rely on what was previously considered a heuristic of using the expert demonstrator as an estimate of the future cost-to-go. The resulting good performance can be understood as a consequence of this heuristic being a special case of \ctgd where the online Frank-Wolfe algorithm \cite{Hazan2012projection} is used to choose policies. Moreover, stochastic mixing is but one of several approaches to achieving good online performance and deterministic variants have proven more effective in practice. \cite{Ross11b}

The resulting algorithms make suggestions for batch approaches as well: they suggest, for instance, that approximate policy iteration procedures (as well as imitation learning ones) are likely to be more stable and effective if they train not only on the cost-to-go of the most recent policy but also on previous policies. At first this may seem counter-intuitive, however, it prevents the oscillations and divergences that at times plague batch approximate dynamic programming algorithms by ensuring that each learned policy is good across many states.

From a broad point of view, this work forms a piece of a growing picture that  online algorithms and no-regret analyses-- in contrast with the traditional \emph{i.i.d.} or batch analysis--  are important for understanding learning with application to control and decision making \cite{Ross13,Ross12,Ross11b,Ross11}. At first glance, online learning seems concerned with a very different adversarial setting. By understanding these methods as attempting to ensure both good performance and robust, stable learning across iterations \cite{Ross12b,Saha2012interplay}, they become a natural tool for understanding the dynamics of interleaving learning and execution when our primary concern is generalization performance.
\todo{JAB: rewrite}

\paragraph{Limitations.} It is important to note that any method relying on cost-to-go estimates can be impractical as collecting each estimate for a single state-action pair may involve executing an entire trajectory. In many settings, minimizing imitation loss with \DAGGER \cite{Ross11}, is more practical as we can observe the action chosen by the expert in \emph{every} visited state along a trajectory and thus collect $T$ data points per trajectory instead of single one. This is less crucial in structured prediction settings where the cost-to-go of the expert may often be quickly computed which has lead to the success of the heuristic analyzed here. A potential combination of the two approaches, where first simple imitation loss minimization provides a reasonable policy, and then this is refined using \ctgd (e.g. through additional gradient descent steps) thus using fewer (expensive) iterations.

In the reinforcement learning setting, the bound provided is as strong as that provided by \cite{Bagnell03,Kakade2003thesis} for an arbitrary policy class.
However, as $T Q_{\max}$ is generally $O(T^2)$, this only provides meaningful guarantees when $d^t_{\pi'}$ is very close to $\nu_t$ (on average over time $t$).
Previous methods like \cite{Bagnell03,Kakade02,Scherrer2014approximate} provide a much stronger, \emph{multiplicative} error guarrantee when we consider competing against the bayes optimal policy in a fully observed MDP. It is not obvious how the current algorithm and analysis can extend to that variant of the bound.

\paragraph{Future Work.} Much work remains to be done: there are a wide variety of no-regret learners and their practical trade-offs are almost completely open. Future work must explore this set to identify which methods are most effective in practice.






\small{
\bibliography{biblio}
\bibliographystyle{plain}
}
\newpage
\section*{Appendix: Proofs and Detailed Bounds}

In this appendix, we provide the proofs and detailed analysis of the algorithms for imitation learning and reinforcement learning provided in the main document. 

\section*{Lemmas}

We begin with a classical and useful general lemma that is needed for bounding the expected loss under different distributions. This will be used several times throughout. Here this will be useful for bounding the expected loss under the state distribution of $\hat{\pi}$ (which optional queries the expert a fraction of the time during it's execution) in terms of the expected loss under the state distribution of $\pi_i$:

\begin{lemma} \label{lemExpL1}
Let $P$ and $Q$ be any distribution over elements $x \in \mathcal{X}$, and $f : \mathcal{X} \rightarrow \mathbb{R}$, any bounded function such that $f(x) \in [a,b]$ for all $x \in \mathcal{X}$. Let the range $r =b-a$. Then $|\mathbb{E}_{x \sim P}[f(x)] - \mathbb{E}_{x \sim Q}[f(x)]| \leq \frac{r}{2} ||P-Q||_1$
\end{lemma}
\begin{proof}
We provide the proof for $\mathcal{X}$ discrete, a similar argument can be carried for $\mathcal{X}$ continuous, using integrals instead of sums.
\begin{displaymath}
\begin{array}{rl}
\multicolumn{2}{l}{|\mathbb{E}_{x \sim P}[f(x)] - \mathbb{E}_{x \sim Q}[f(x)]|}\\
= & | \sum_x P(x) f(x) - Q(x) f(x) |\\
= & | \sum_x f(x) (P(x) - Q(x)) |\\
\end{array}
\end{displaymath}
Additionally, since for any real $c \in \mathbb{R}$, $\sum_x P(x) c = \sum_{x} Q(x) c$, then we have for any $c$:
\begin{displaymath}
\begin{array}{rl}
\multicolumn{2}{l}{| \sum_x f(x) (P(x) - Q(x)) |}\\
= & | \sum_x (f(x) - c) (P(x) - Q(x)) |\\
\leq & \sum_x |f(x)-c| |P(x) - Q(x)|\\
\leq & \max_x |f(x)-c|  \sum_x |P(x) - Q(x)|\\
= & \max_x |f(x)-c|  ||P-Q||_1\\
\end{array}
\end{displaymath}
This holds for all $c \in \mathbb{R}$. This upper bound is minimized for $c = a + \frac{r}{2}$, making $\max_x |f(x)-c| \leq \frac{r}{2}$. This proves the lemma. 
\end{proof}

The $L_1$ distance between the distribution of states encountered by $\hat{\pi}_i$, the policy chosen by the online learner, and $\pi_i$, the policy used to collect data that continues to execute the expert's actions with probability $\beta_i$ is bounded as follows:%
\begin{lemma} \label{lemL1Dist}
$||d_{\pi_i} - d_{\hat{\pi}_i}||_1 \leq 2 \min( 1, T \beta_i )$.
\end{lemma}
\begin{proof}
Let $d$ the distribution of states over $T$ steps conditioned on $\pi_i$ picking the expert $\pi^*$ at least once over $T$ steps. Since $\pi_i$ always executes $\hat{\pi}_i$ (never executes the expert action) over $T$ steps with probability $(1-\beta_i)^T$ we have $d_{\pi_i} = (1-\beta_i)^T d_{\hat{\pi}_i} + (1-(1-\beta_i)^T) d$. Thus
\begin{displaymath}
\begin{array}{rl}
\multicolumn{2}{l}{||d_{\pi_i} - d_{\hat{\pi}_i}||_1}\\
= & (1-(1-\beta_i)^T) ||d - d_{\hat{\pi}_i}||_1\\
\leq & 2 (1-(1-\beta_i)^T)\\
\leq & 2 T \beta_i \\
\end{array}
\end{displaymath}
The last inequality follows from the fact that $(1-\beta)^T \geq 1 - \beta T$ for any $\beta \in [0,1]$. Finally, since for any 2 distributions $p$, $q$, we always have $||p - q||_1 \leq 2$, then $||d_{\pi_i} - d_{\hat{\pi}_i}||_1 \leq 2 \min( 1, T \beta_i )$.
\end{proof}

Below we use the \emph{performance difference lemma} \cite{Bagnell03,Kakade02,Kakade2003thesis} that is useful to bound the change in total cost-to-go. This general result bounds the difference in performance of any two policies.  We present this results and its proof here for completeness.

\begin{lemma} \label{lemPerfDiff}
Let $\pi$ and $\pi'$ be any two policy and denote $V'_t$ and $Q'_t$ the $t$-step value function and $Q$-value function of policy $\pi'$ respectively, then:
\begin{displaymath}
J(\pi) - J(\pi') = T \Exp_{t \sim U(1:T), s \sim d^t_{\pi}}[Q'_{T-t+1}(s,\pi) - V'_{T-t+1}(s)]
\end{displaymath}
for $U(1:T)$ the uniform distribution on the set $\{1,2,\dots,T\}$.
\end{lemma}
\begin{proof}
Let $\pi_{t}$ denote the non-stationary policy that executes $\pi$ in the first $t$ time steps, and then switches to execute $\pi'$ at time $t+1$ to $T$. Then we have $J(\pi) = J(\pi_{T})$ and $J(\pi') = J(\pi_0)$. Thus:
\begin{displaymath}
\begin{array}{rl}
\multicolumn{2}{l}{J(\pi) - J(\pi')}\\
= & \sum_{t=1}^T [J(\pi_t) - J(\pi_{t-1})]\\
= & \sum_{t=1}^T [ \Exp_{s \sim d^t_{\pi}}[Q'_{T-t+1}(s,\pi) - V'_{T-t+1}(s)] ]\\
= & T \Exp_{t \sim U(1:T),s \sim d^t_{\pi}}[Q'_{T-t+1}(s,\pi) - V'_{T-t+1}(s)]\\
\end{array}
\end{displaymath}
\end{proof}

\section*{\ctgd Reduction Analysis}

Let $\epsilon_{\textrm{class}} = \min_{\pi \in \Pi} \frac{1}{N} \sum_{i=1}^N \Exp_{t \sim U(1:T), s \sim d^t_{\pi_i}}[ Q_{T-t+1}^*(s,\pi) - \min_a Q_{T-t+1}^*(s,a) ]$ denote the minimum expected cost-sensitive classification regret achieved by policies in the class $\Pi$ on all the data over the $N$ iterations of training. Denote the online learning average regret on the cost-to-go examples of the sequence of policies chosen by \ctgd, $\epsilon_{\textrm{regret}} = \frac{1}{N} [ \sum_{i=1}^N \ell_i(\hat{\pi}_i) - \min_{\pi \in \Pi}  \sum_{i=1}^N \ell_i(\pi) ]$, where $\ell_i(\pi) = \Exp_{t \sim U(1:T), s \sim d^t_{\pi_i}}[Q_{T-t+1}^*(s,\pi)]$. Assume the cost-to-go of the expert $Q^*$ is non-negative and bounded by $Q^*_{\max}$, and that $\beta_i$ are chosen such that $\beta_i \leq (1-\alpha)^{i-1}$ for some $\alpha$. Then we have the following:

\begin{theorem} 
After $N$ iterations of \ctgd:
\begin{displaymath}
J(\hat{\pi}) \leq J(\overline{\pi}) \leq J(\pi^*) + T[\epsilon_{\textrm{class}} + \epsilon_{\textrm{regret}}] + O\left(\frac{Q^*_{\max} T \log T}{\alpha N}\right).
\end{displaymath}
Thus if a no-regret online algorithm is used to pick the sequence of policies $\hat{\pi}_{1:N}$, then as the number of iterations $N \rightarrow \infty$:
\begin{displaymath}
\lim_{N \rightarrow \infty} J(\overline{\pi}) \leq J(\pi^*) + T \epsilon_{\textrm{class}}
\end{displaymath}
\end{theorem}
\begin{proof}
For every policy $\hat{\pi}_i$, we have:
\begin{displaymath}
\begin{array}{rl}
\multicolumn{2}{l}{J(\hat{\pi}_i) - J(\pi^*)}\\
= & T \Exp_{t \sim U(1:T),s \sim d^t_{\hat{\pi}_i}}[Q^*_{T-t+1}(s,\hat{\pi}_i) - V^*_{T-t+1}(s)]\\
= & \sum_{t=1}^T \Exp_{s \sim d^t_{\hat{\pi}_i}}[Q^*_{T-t+1}(s,\hat{\pi}_i) - V^*_{T-t+1}(s)]\\
\leq & \sum_{t=1}^T \Exp_{s \sim d^t_{\pi_i}}[Q^*_{T-t+1}(s,\hat{\pi}_i) - V^*_{T-t+1}(s)] +  Q^*_{\max} \sum_{t=1}^T ||d^t_{\pi_i} - d^t_{\hat{\pi}_i}||_1\\
\leq & \sum_{t=1}^T \Exp_{s \sim d^t_{\pi_i}}[Q^*_{T-t+1}(s,\hat{\pi}_i) - V^*_{T-t+1}(s)] +  2 Q^*_{\max} \sum_{t=1}^T \min(1,t \beta_i)\\
\leq & \sum_{t=1}^T \Exp_{s \sim d^t_{\pi_i}}[Q^*_{T-t+1}(s,\hat{\pi}_i) - V^*_{T-t+1}(s)] +  2 T Q^*_{\max} \min(1, T \beta_i)\\
= & T \Exp_{t \sim U(1:T),s \sim d^t_{\pi_i}}[Q^*_{T-t+1}(s,\hat{\pi}_i) - V^*_{T-t+1}(s)] +  2 T Q^*_{\max} \min(1, T \beta_i)\\
\end{array}
\end{displaymath}
where we use lemma \ref{lemPerfDiff} in the first equality, lemma \ref{lemExpL1} in the first inequality, and a similar argument to lemma \ref{lemL1Dist} for the second inequality.

Since $\beta_i$ are non-increasing, define $n_\beta$ the largest $n \leq N$ such that $\beta_n > 1/T$. Then:
\begin{displaymath}
\begin{array}{rl}
\multicolumn{2}{l}{J(\overline{\pi}) - J(\pi^*)}\\
= & \frac{1}{N} \sum_{i=1}^N [J(\hat{\pi}_i) - J(\pi^*)]\\
\leq & \frac{1}{N} \sum_{i=1}^N [ T \Exp_{t \sim U(1:T),s \sim d^t_{\pi_i}}[Q^*_{T-t+1}(s,\hat{\pi}_i) - V^*_{T-t+1}(s)] +  2 T Q^*_{\max} \min(1, T \beta_i) ]\\
=  & T [\min_{\pi \in \Pi} \frac{1}{N} \sum_{i=1}^N \Exp_{t \sim U(1:T),s \sim d^t_{\pi_i}}[Q^*_{T-t+1}(s,\pi)-V^*_{T-t+1}(s)]] + T \epsilon_{\textrm{regret}} \\
& + \frac{2 T Q^*_{\max}}{N} [ n_\beta + T \sum_{i=n_\beta+1}^N \beta_i ]\\
\leq & T [\min_{\pi \in \Pi} \frac{1}{N} \sum_{i=1}^N \Exp_{t \sim U(1:T),s \sim d^t_{\pi_i}}[Q^*_{T-t+1}(s,\pi)-\min_a Q^*_{T-t+1}(s,a)] + T \epsilon_{\textrm{regret}} \\
& + \frac{2 T Q^*_{\max}}{N} [ n_\beta + T \sum_{i=n_\beta+1}^N \beta_i ]\\
=  & T \epsilon_{\textrm{class}} + T \epsilon_{\textrm{regret}} + \frac{2 T Q^*_{\max}}{N} [ n_\beta + T \sum_{i=n_\beta+1}^N \beta_i ]\\
\end{array}
\end{displaymath}

Again, $J(\hat{\pi}) \leq J(\overline{\pi})$ since the minimum is always better than the average, i.e. $\min_i J(\hat{\pi}_i) \leq \frac{1}{N} \sum_{i=1}^N J(\hat{\pi}_i)$. Finally, we have that when $\beta_i = (1-\alpha)^{i-1}$, $[ n_\beta + T \sum_{i=n_\beta+1}^N \beta_i ] \leq \frac{\log(T) + 2}{\alpha}$. This proves the first part of the theorem.

The second part follows immediately from the fact that  $\epsilon_{\textrm{regret}} \rightarrow 0$ as $N \rightarrow \infty$, and similarly for the extra term $O\left(\frac{Q^*_{\max} T \log T}{\alpha N}\right)$.
\end{proof}

\section*{Finite Sample \ctgd with Q-function approximation}

We here consider the finite sample case where actions are explored uniformly randomly and the reduction of cost-sensitive classification to squared loss regression is used. We consider learning an estimate Q-value function $\hat{Q}$ of the expert's cost-to-go, and we consider a general case where the cost-to-go predictions may depend on features $f(s,a,t)$ of the state $s$, action $a$ and time $t$, e.g. $\hat{Q}$ could be a linear regressor s.t. $\hat{Q}_{T-t+1}(s,a) = w^{\top} f(s,a,t)$ is the estimate of the cost-to-go $Q^*_{T-t+1}(s,a)$, and $w$ are the parameters of the linear regressor we learn. Given such estimates $\hat{Q}$, we consider executing the policy $\hat{\pi}$, such that in state $s$ at time $t$, $\hat{\pi}(s,t) = \min_{a \in A} \hat{Q}_{T-t+1}(s,a)$.

\begin{theorem}
After $N$ iterations of \ctgd, collecting $m$ regression examples $(s,a,t,Q)$ per iteration, guarantees that with probability at least 1-$\delta$:
\begin{displaymath}
J(\hat{\pi}) \leq J(\overline{\pi}) \leq J(\pi^*) + 2 \sqrt{|A|}T \sqrt{\hat{\epsilon}_{\textrm{class}} + \hat{\epsilon}_{\textrm{regret}} + O(\sqrt{\log(1/\delta)/Nm})} + O\left(\frac{Q_{\max} T \log T}{\alpha N}\right).
\end{displaymath}
Thus if a no-regret online algorithm is used to pick the sequence of regressors $\hat{Q}_{1:N}$, then as the number of iterations $N \rightarrow \infty$, with probability 1:
\begin{displaymath}
\lim_{N \rightarrow \infty} J(\overline{\pi}) \leq J(\pi^*) + 2 \sqrt{|A|}T \sqrt{\hat{\epsilon}_{\textrm{class}}}
\end{displaymath}
\end{theorem}
\begin{proof}
Consider $\tilde{\pi}$, the bayes-optimal non-stationary policy that minimizes loss on the cost-to-go examples. That is, $\tilde{\pi}(s,t) = \min_{a \in A} Q_{T-t+1}^*(s,a)$, i.e. it picks the action with minimum expected expert cost-to-go conditioned on being in state $s$ and time $t$. Additionally, given the observed noisy Q-values from each trajectory, the bayes-optimal regressor is simply the Q-value function $Q^*$ of the expert that predicts the expected cost-to-go. 

At each iteration $i$, we execute a policy $\hat{\pi}_i$, such that $\hat{\pi}_i(s,t) = \argmin_{a \in A} \hat{Q}^i_{T-t+1}(s,a)$, where $\hat{Q}^i$ is the current regressor at iteration $i$ from the base online learner. The cost-sensitive regret of policy $\hat{\pi}_i$, compared to $\tilde{\pi}$, can be related to the regression regret of $\hat{Q}_i$ as follows:

Consider any state $s$ and time $t$. Let $\hat{a}_i = \hat{\pi}_i(s,t)$ and consider the action $a'$ of any other policy. We have that:

\begin{displaymath}
\begin{array}{rl}
\multicolumn{2}{l}{Q^*_{T-t+1}(s,\hat{a}_i) - Q^*_{T-t+1}(s,a)}\\
\leq & \hat{Q}^i_{T-t+1}(s,\hat{a}_i) - \hat{Q}^i_{T-t+1}(s,a') + Q^*_{T-t+1}(s,\hat{a}_i) - \hat{Q}^i_{T-t+1}(s,\hat{a}_i) + \hat{Q}^i_{T-t+1}(s,a') -  Q^*_{T-t+1}(s,a')\\
\leq & Q^*_{T-t+1}(s,\hat{a}_i) - \hat{Q}^i_{T-t+1}(s,\hat{a}_i) + \hat{Q}^i_{T-t+1}(s,a') -  Q^*_{T-t+1}(s,a')\\
\leq & 2 \max_{a \in A} |Q^*_{T-t+1}(s,a) - \hat{Q}^i_{T-t+1}(s,a)|\\
\end{array}
\end{displaymath}

Additionally, for any joint distribution $D$ over $(s,t)$, and $U(A)$ the uniform distribution over actions, we have that:

\begin{displaymath}
\begin{array}{rl}
\multicolumn{2}{l}{( \mathbb{E}_{(s,t) \sim D}[ \max_{a \in A} |Q^*_{T-t+1}(s,a) - \hat{Q}^i_{T-t+1}(s,a)|] )^2}\\
\leq & \mathbb{E}_{(s,t) \sim D}[ \max_{a \in A} |Q^*_{T-t+1}(s,a) - \hat{Q}^i_{T-t+1}(s,a)|^2 ]\\
\leq & \mathbb{E}_{(s,t) \sim D}[ \sum_{a \in A} |Q^*_{T-t+1}(s,a) - \hat{Q}^i_{T-t+1}(s,a)|^2 ]\\
= & |A| \mathbb{E}_{(s,t) \sim D, a \sim U(A)}[ |Q^*_{T-t+1}(s,a) - \hat{Q}^i_{T-t+1}(s,a)|^2 ]\\
\end{array}
\end{displaymath}

Thus we obtain that for every $\hat{\pi}_i$:

\begin{displaymath}
\begin{array}{rl}
\multicolumn{2}{l}{\Exp_{t \sim U(1:T), s \sim d^t_{\pi_i}}[Q^*_{T-t+1}(s,\hat{\pi}_i) - Q^*_{T-t+1}(s,\tilde{\pi})]}\\
\leq & 2 \Exp_{t \sim U(1:T), s \sim d^t_{\pi_i}}[\max_{a \in A} |Q^*_{T-t+1}(s,a) - \hat{Q}^i_{T-t+1}(s,a)|]\\
\leq & 2 \sqrt{|A|} \sqrt{\mathbb{E}_{t \sim U(1:T), s \sim d^t_{\pi_i}, a \sim U(A)}[ |Q^*_{T-t+1}(s,a) - \hat{Q}^i_{T-t+1}(s,a)|^2 ]}\\
\end{array}
\end{displaymath}

Thus

\begin{displaymath}
\begin{array}{rl}
\multicolumn{2}{l}{J(\overline{\pi}) - J(\pi^*)}\\
= & \frac{T}{N} \sum_{i=1}^N \Exp_{t \sim U(1:T), s \sim d^t_{\hat{\pi}_i}}[Q^*_{T-t+1}(s,\hat{\pi}_i) - Q^*_{T-t+1}(s,\pi^*)]\\
\leq & \frac{T}{N} \sum_{i=1}^N \Exp_{t \sim U(1:T), s \sim d^t_{\pi_i}}[Q^*_{T-t+1}(s,\hat{\pi}_i) - Q^*_{T-t+1}(s,\pi^*)] +  \frac{2 T Q^*_{\max}}{N} [ n_\beta + T \sum_{i=n_\beta+1}^N \beta_i ]\\
\leq & \frac{T}{N} \sum_{i=1}^N  \Exp_{t \sim U(1:T), s \sim d^t_{\pi_i}}[Q^*_{T-t+1}(s,\hat{\pi}_i) - Q^*_{T-t+1}(s,\tilde{\pi})] +  \frac{2 T Q^*_{\max}}{N} [ n_\beta + T \sum_{i=n_\beta+1}^N \beta_i ]\\
\leq & \frac{2 \sqrt{|A|} T}{N}  \sum_{i=1}^N \sqrt{\mathbb{E}_{t \sim U(1:T), s \sim d^t_{\pi_i}, a \sim U(A)}[ |Q^*_{T-t+1}(s,a) - \hat{Q}^i_{T-t+1}(s,a)|^2 ]} \\
& +  \frac{2 T Q^*_{\max}}{N} [ n_\beta + T \sum_{i=n_\beta+1}^N \beta_i ]\\
\leq & 2 \sqrt{|A|} T  \sqrt{\frac{1}{N} \sum_{i=1}^N \mathbb{E}_{t \sim U(1:T), s \sim d^t_{\pi_i}, a \sim U(A)}[ |Q^*_{T-t+1}(s,a) - \hat{Q}^i_{T-t+1}(s,a)|^2 ]} \\
& +  \frac{2 T Q^*_{\max}}{N} [ n_\beta + T \sum_{i=n_\beta+1}^N \beta_i ]\\
\end{array}
\end{displaymath}

Now in state $s$ at time $t$, when performing $a$ and then following the expert, consider the distribution $d_{s,a,t}$ over observed cost-to-go $Q$, such that $\mathbb{E}_{Q \sim d_{s,a,t}}[Q] = Q^*_{T-t+1}(s,a)$.

For any regressor $\hat{Q}$, define the expected squared loss in predictions of the observed cost-to-go at iteration $i$ as $\ell_i(\hat{Q}) = \mathbb{E}_{t \sim U(1:T), s \sim d^t_{\pi_i}, a \sim U(A), Q \sim d_{s,t,a} }[ |Q - \hat{Q}_{T-t+1}(s,a)|^2 ]$. Then since for any random variable $X$ with mean $\mu$, if we have an estimate $\hat{\mu}$ of the mean, $|\hat{\mu}-\mu|^2 = \mathbb{E}_x[(x-\hat{\mu})^2 - (x-\mu)^2]$, we have that:

\begin{displaymath}
\frac{1}{N} \sum_{i=1}^N \mathbb{E}_{t \sim U(1:T), s \sim d^t_{\pi_i}, a \sim U(A)}[ |Q^*_{T-t+1}(s,a) - \hat{Q}^i_{T-t+1}(s,a)|^2 ]
= \frac{1}{N}  \sum_{i=1}^N \ell_i(\hat{Q}^i) - \ell_i(Q^*)
\end{displaymath}

Now, in the finite sample case, consider collecting $m$ samples at each iteration $i$: $\{ (s_{ij},a_{ij}, t_{ij}, Q_{ij}) \}_{j=1}^m$. The expected squared loss $\ell_i$ is estimated as $\hat{\ell}_i(\hat{Q}) = \frac{1}{m} \sum_{j=1}^m (\hat{Q}_{T-t_{ij}+1}(s_{ij},a_{ij}) - Q_{ij})^2$, and the no-regret algorithm is run on the estimated loss $\hat{\ell}_i$.

Define $Y_{i,j} = \ell_i(\hat{Q}^i) - (\hat{Q}^i_{T-t_{ij}+1}(s_{ij},a_{ij}) - Q_{ij})^2 - \ell_i(Q^*) + (Q^*_{T-t_{ij}+1}(s_{ij},a_{ij}) - Q_{ij})^2$, the difference between the expected squared loss and the empirical square loss at each sample for both $\hat{Q}^i$ and $Q^*$. Conditioned on previous trajectories, each $Y_{i,j}$ has expectation 0. Then the sequence of random variables $X_{km+l} = \sum_{i=1}^k \sum_{j=1}^m Y_{i,j} + \sum_{j=1}^l Y_{(k+1),j}$, for $k \in \{0,1,2,\dots,N-1\}$ and $l \in \{1,2,\dots,m\}$, forms a martingale, and if the squared loss at any sample is bounded by $\ell_{\max}$, we obtain that $|X_{i} - X_{i+1}| \leq 2 \ell_{\max}$. By Azuma-Hoeffding's inequality, this implies that with probability at least $1-\delta$, $\frac{1}{Nm}X_{Nm} \leq 2 \ell_{\max} \sqrt{\frac{2\log(1/\delta)}{Nm}}$. 

Denote the empirical average online regret on the training squared loss $\hat{\epsilon}_{\textrm{regret}} =  \frac{1}{N} \sum_{i=1}^N \hat{\ell}_i(\hat{Q}^i) -  \min_{\hat{Q} \in \mathcal{Q}} \frac{1}{N} \sum_{i=1}^N \hat{\ell}_i(\hat{Q})$. Let $\tilde{Q}^*$ be the bayes-optimal regressor on the finite training data, and define the empirical regression regret of the best regressor in the class as $\hat{\epsilon}_{\textrm{class}} =  \min_{\hat{Q} \in \mathcal{Q}} \frac{1}{N} \sum_{i=1}^N [\hat{\ell}_i(\hat{Q}) - \hat{\ell}_i(\tilde{Q}^*)]$.
 
Then we obtain that with probability at least $1-\delta$:

\begin{displaymath}
\begin{array}{rl}
\multicolumn{2}{l}{\frac{1}{N}  \sum_{i=1}^N \ell_i(\hat{Q}^i) - \ell_i(Q^*)}\\
= & \frac{1}{N} \sum_{i=1}^N \hat{\ell}_i(\hat{Q}^i) - \hat{\ell}_i(Q^*) + \frac{1}{Nm} X_{Nm}\\
\leq & \frac{1}{N} \sum_{i=1}^N \hat{\ell}_i(\hat{Q}^i) - \hat{\ell}_i(Q^*) + 2 \ell_{\max} \sqrt{\frac{2\log(1/\delta)}{Nm}}\\
\leq & \min_{\hat{Q} \in \mathcal{Q}} \frac{1}{N} \sum_{i=1}^N [\hat{\ell}_i(\hat{Q}) - \hat{\ell}_i(Q^*)] + \hat{\epsilon}_{\textrm{regret}} + 2 \ell_{\max} \sqrt{\frac{2\log(1/\delta)}{Nm}}\\
\leq & \hat{\epsilon}_{\textrm{class}} + \hat{\epsilon}_{\textrm{regret}} + 2 \ell_{\max} \sqrt{\frac{2\log(1/\delta)}{Nm}}\\
\end{array}
\end{displaymath}

where the last inequality follows from the fact that $\sum_{i=1}^N \hat{\ell}_i(\tilde{Q}^*) \leq \sum_{i=1}^N \hat{\ell}_i(Q^*)$.

Combining with the above, we obtain that with probability at least $1-\delta$:

\begin{displaymath}
J(\overline{\pi}) - J(\pi^*) \leq 2 \sqrt{|A|} T  \sqrt{\hat{\epsilon}_{\textrm{class}} + \hat{\epsilon}_{\textrm{regret}} + 2 \ell_{\max} \sqrt{\frac{2\log(1/\delta)}{Nm}}} + \frac{2 T Q^*_{\max}}{N} [ n_\beta + T \sum_{i=n_\beta+1}^N \beta_i ]
\end{displaymath}
\end{proof}

\section*{\rldagger Reduction Analysis}

We here provide the proof of the result for \rldagger, sampled from state exploration distributions $\nu_{1:T}$.

To analyze this version, we begin with an alternate version of the performance difference lemma (lemma  \ref{lemPerfDiff}) presented before:
\begin{lemma} \label{lemPerfDiff2}
Let $\pi$ and $\pi'$ be any two policy and denote $V_t$ and $Q_t$ the $t$-step value function and $Q$-value function of policy $\pi$ respectively, then:
\begin{displaymath}
J(\pi) - J(\pi') = T \Exp_{t \sim U(1:T), s \sim d^t_{\pi'}}[V_{T-t+1}(s) - Q_{T-t+1}(s,\pi')]
\end{displaymath}
for $U(1:T)$ the uniform distribution on the set $\{1,2,\dots,T\}$.
\end{lemma}
\begin{proof}
By applying lemma \ref{lemPerfDiff} to $J(\pi')-J(\pi)$, we obtain:
$$J(\pi')-J(\pi) = T \Exp_{t \sim U(1:T), s \sim d^t_{\pi'}}[Q_{T-t+1}(s,\pi')-V_{T-t+1}(s)]$$
This proves the lemma.
\end{proof}

Now denote the loss $L_n$ used by the online learner at iteration $n$, s.t.:
$$L_n(\pi) = \mathbb{E}_{t \sim U(1:T), s \sim \nu_t}[Q^{\hat{\pi}_n}_{T-t+1}(s,\pi)].$$
and $\epsilon_{\textrm{regret}}$ the average regret after the $N$ iterations of \rldagger:
$$\epsilon_{\textrm{regret}} = \frac{1}{N} \sum_{i=1}^N L_i(\pi_i) - \min_{\pi \in \Pi} \frac{1}{N} \sum_{i=1}^N L_i(\pi).$$

For any policy $\pi \in \Pi$, denote the average $L_1$ distance between $\nu_t$ and $d^t_{\pi}$ over time steps $t$ as:
$$D(\nu,\pi) = \frac{1}{T} \sum_{t=1}^T ||\nu_t - d^t_\pi||_1.$$

Assume the cost-to-go of the learned policies $\pi_1,\pi_2,\dots,\pi_N$ are non-negative and bounded by $Q_{\max}$, for any state $s$, action $a$ and time $t$ (in the worst case this is $T C_{\max}$). Denote $\hat{\pi}$ the best policy found by \rldagger over the iterations, and $\overline{\pi}$ the uniform mixture policy over $\pi_{1:N}$ defined as before. Then we have to following guarantee with this version of \rldagger with learner's cost-to-go:

\begin{theorem}
For any $\pi' \in \Pi$:
\begin{displaymath}
J(\hat{\pi}) \leq J(\overline{\pi}) \leq J(\pi') + T \epsilon_{\textrm{regret}} + T Q_{\max} D(\nu,\pi')
\end{displaymath}
Thus, if a no-regret online cost-sensitive classification algorithm is used, then:
\begin{displaymath}
\lim_{N \rightarrow \infty} J(\overline{\pi}) \leq J(\pi') + T Q_{\max} D(\nu,\pi')
\end{displaymath}
\end{theorem}
\begin{proof}
Let $Q^i_t$ denote the $t$-step $Q$-value function of policy $\hat{\pi}_i$. Then for every $\hat{\pi}_i$ we have:
\begin{displaymath}
\begin{array}{rl}
\multicolumn{2}{l}{J(\hat{\pi}_i) - J(\pi')}\\
= & T \Exp_{t \sim U(1:T),s \sim d^t_{\pi'}}[Q^{i}_{T-t+1}(s,\hat{\pi}_i) - Q^i_{T-t+1}(s,\pi')]\\
= & \sum_{t=1}^T \Exp_{s \sim d^t_{\pi'}}[Q^{i}_{T-t+1}(s,\hat{\pi}_i) - Q^i_{T-t+1}(s,\pi')]\\
\leq & \sum_{t=1}^T \Exp_{s \sim \nu_t}[Q^{i}_{T-t+1}(s,\hat{\pi}_i) - Q^i_{T-t+1}(s,\pi')] + Q_{\max} \sum_{t=1}^T ||\nu_t - d^t_{\pi'}||_1\\
= & T \Exp_{t \sim U(1:T),s \sim \nu_t}[Q^{i}_{T-t+1}(s,\hat{\pi}_i) - Q^i_{T-t+1}(s,\pi')] + T Q_{\max} D(\nu,\pi')\\
\end{array}
\end{displaymath}
where we use lemma \ref{lemPerfDiff2} in the first equality, and lemma \ref{lemExpL1} in the first inequality.

Thus:
\begin{displaymath}
\begin{array}{rl}
\multicolumn{2}{l}{J(\overline{\pi}) - J(\pi')}\\
= & \frac{1}{N} \sum_{i=1}^N [J(\hat{\pi}_i) - J(\pi')]\\
\leq & \frac{1}{N} \sum_{i=1}^N [ T \Exp_{t \sim U(1:T),s \sim \nu_t}[Q^{i}_{T-t+1}(s,\hat{\pi}_i) - Q^i_{T-t+1}(s,\pi')] + T Q_{\max} D(\nu,\pi') ]\\
\leq & T \frac{1}{N} \sum_{i=1}^N \Exp_{t \sim U(1:T),s \sim \nu_t}[Q^{i}_{T-t+1}(s,\hat{\pi}_i)] - T \min_{\pi \in \Pi} \frac{1}{N} \sum_{i=1}^N \Exp_{t \sim U(1:T),s \sim \nu_t}[Q^i_{T-t+1}(s,\pi)]\\
& + T Q_{\max} D(\nu,\pi')\\
=  & T \epsilon_{\textrm{regret}} + T Q_{\max} D(\nu,\pi')\\
\end{array}
\end{displaymath}

Again, $J(\hat{\pi}) \leq J(\overline{\pi})$ since the minimum is always better than the average, i.e. $\min_i J(\hat{\pi}_i) \leq \frac{1}{N} \sum_{i=1}^N J(\hat{\pi}_i)$. This proves the first part of the theorem.

The second part follows immediately from the fact that  $\epsilon_{\textrm{regret}} \rightarrow 0$ as $N \rightarrow \infty$.
\end{proof}


\end{document}